\documentclass{article}

\usepackage[preprint]{acl}

\usepackage{times}
\usepackage{latexsym}
\usepackage[T1]{fontenc}
\usepackage[utf8]{inputenc}
\usepackage{microtype}

\usepackage{graphicx}
\graphicspath{{images/}}
\usepackage{amsmath,amssymb,amsfonts}
\usepackage{amsthm}
\usepackage{booktabs}
\usepackage{multirow}
\usepackage{algorithm}
\usepackage{algorithmic}
\usepackage{xcolor}
\usepackage{enumitem}
\usepackage{placeins}
\usepackage{subcaption}
\usepackage{url}
\usepackage{cleveref}

\newcommand{\ours}{\textsc{Purge}}
\newcommand{\klr}{\texttt{kl\_retain}}
\newcommand{\eg}{\textit{e.g.}}

\newcommand{\KL}{\mathrm{KL}}
\newcommand{\CE}{\mathrm{CE}}

\newcommand{\Df}{\mathcal{D}_f}
\newcommand{\Dr}{\mathcal{D}_r}
\newcommand{\gf}{g_{\mathrm{f}}}
\newcommand{\gr}{g_{\mathrm{r}}}
\newcommand{\gproj}{\tilde{g}_{\mathrm{f}}}
\newcommand{\lossf}{\mathcal{L}_{\mathrm{f}}}
\newcommand{\lossr}{\mathcal{L}_{\mathrm{r}}}

\newcommand{\losskd}{\mathcal{L}_{\mathrm{KD}}}

\newtheorem{theorem}{Theorem}
\newtheorem{proposition}{Proposition}

\newtheorem{definition}{Definition}
\newtheorem{remark}{Remark}

\title{PURGE: Projected Unlearning via Retain-Guided Erasure}

\author{
  \textbf{Vedant Jawandhia\textsuperscript{1}},
  \textbf{Daksh Ahuja\textsuperscript{1}} \\
  \textbf{Ghufran Alam Siddiqui\textsuperscript{1}},
  \textbf{Prashant Trivedi\textsuperscript{1}},
  \textbf{Yash Sinha\textsuperscript{1}},
  \textbf{Pratik Narang\textsuperscript{1}} \\
  \textsuperscript{1}BITS Pilani, Pilani Campus, India
}

\begin{document}
\maketitle

\begin{abstract}
We propose \ours{}, a machine unlearning algorithm built on a simple but an under-exploited observation: continual learning (CL) and machine unlearning (MU) which are fundamentally dual problems. CL tries to learn new tasks without forgetting old ones; MU tries to erase specific data without hurting retained performance representing the same underlying tension in opposite directions.
\ours{} leverages this duality by adapting \emph{gradient projection} from A-GEM~\cite{chaudhry2019agem} so that every unlearning step is constrained to not increase the retain-set loss.
On top of this, it performs \emph{multi-layer representation erasure}, pushing forget-set activations in intermediate layers towards the retain distribution to remove information from hidden representations rather than just suppressing it at the output.
A key design choice is the \emph{retain-confusion target}: rather than pushing forget outputs toward the uniform distribution, which we found to be surprisingly easy for membership inference attacks to detect, we instead target the model's natural confusion pattern on retain data.
This makes the unlearned model hard to distinguish from one retrained from scratch.
Two self-regulating stopping criteria (a retain-loss budget and a forget-accuracy target) let the algorithm decide on its own when to stop, removing the need for manual epoch tuning.
In experiments on five datasets (CIFAR-10, MNIST, SVHN, STL10, PathMNIST) across 22 class-level forgetting tasks, \ours{} consistently keeps retain accuracy above 96\% while achieving MIA AUROC close to 0.5 (the ideal), outperforming gradient ascent, KL-uniform, and several published baselines on the privacy--utility frontier.
\end{abstract}

\section{Introduction}\label{sec:intro}

The right to be forgotten, enshrined in regulations such as the EU General Data Protection Regulation (GDPR)~\cite{voigt2017gdpr}, requires that machine learning systems be capable of removing the influence of specific training data upon request.
Na\"ive retraining from scratch is the gold-standard solution, but it is computationally prohibitive for large models.
\emph{Machine unlearning} (MU)~\cite{bourtoule2021sisa} aims to approximate the effect of retraining without incurring its full cost.

Existing MU methods face a fundamental tension: forgetting aggressively enough to satisfy privacy audits while preserving the model's utility on retained data.

In our evaluation of representative baseline methods spanning gradient-based, saliency-based, distillation-based, and Fisher-based approaches, each method exhibits a characteristic failure mode.
Gradient ascent~\cite{thudi2022unrolling} maximises the loss on forget data but often degrades retain-set accuracy.
Saliency-based methods such as SalUn~\cite{fan2024salun} selectively perturb high-importance weights but do not provide a formal retain-performance guarantee.
Knowledge distillation approaches~\cite{chundawat2023badteacher} anchor the retain distribution but do not constrain the gradient direction.
Fisher-based methods~\cite{golatkar2021scrubbing} are theoretically grounded but require expensive Hessian approximations that are difficult to scale.

The central observation behind our work is that \textbf{machine unlearning is the dual of continual learning (CL)}: CL protects old-task knowledge while learning new tasks, whereas MU protects retained knowledge while erasing specific data, representing the same underlying trade-off in opposite directions.
This connection has been acknowledged by prior work — most notably \citet{kurmanji2023unbounded}, whose SCRUB method draws on continual learning intuitions — but existing approaches use CL as a conceptual analogy rather than directly adapting a CL mechanism at the algorithmic level. To our knowledge, no prior work has instantiated gradient projection from a CL algorithm as the primary constructive tool for retain-safe unlearning.
We build on this insight by adapting A-GEM~\cite{chaudhry2019agem}, which projects new-task gradients onto a half-space that does not increase old-task loss.
In our setting, forget-direction gradients are projected onto the half-space where retain-set loss does not increase, providing a per-step constructive guarantee that existing unlearning methods lack.

Building on this insight, we propose \ours{} (\textbf{P}rojected \textbf{U}nlearning via \textbf{R}etain-\textbf{G}uided \textbf{E}rasure).
The method was developed following limitations observed in an earlier approach (an influence-weighted entropy method, IEWPv2), which exhibited instability in multi-class settings.
\ours{} combines five components:
\begin{enumerate}[nosep,leftmargin=*]
  \item \textbf{Gradient projection} from CL, adapted for retain-safe unlearning;
  \item A \textbf{retain-confusion target} (\klr{}) that encourages forget outputs toward the model's natural confusion distribution on the retain set, achieving MIA AUROC $\approx 0.5$;
  \item \textbf{Multi-layer representation erasure} that removes information from hidden activations, not only outputs;
  \item \textbf{KD-anchored stability} using a frozen pre-unlearning model;
  \item \textbf{Dual self-regulating stopping} (retain-loss budget + FA target) to determine when unlearning is complete.
\end{enumerate}

We evaluate \ours{} on CIFAR-10~\cite{krizhevsky2009cifar}, MNIST~\cite{lecun1998mnist}, SVHN~\cite{netzer2011svhn}, STL10~\cite{coates2011stl10}, and PathMNIST~\cite{yang2023medmnist}, demonstrating strong privacy--utility trade-offs across all five domains.

\section{Related Work}\label{sec:related}

\paragraph{Exact unlearning.}
SISA training~\cite{bourtoule2021sisa} partitions data into shards to enable exact removal by retraining a single shard.
This approach provides provable guarantees but scales poorly with increasing model size and shard count.

\paragraph{Gradient-based approximate unlearning.}
Gradient ascent~\cite{thudi2022unrolling} tries to maximises the loss on forget data.
While simple, it lacks a mechanism to preserve retain-set performance and often leads to severe accuracy degradation (e.g., test accuracy drops to $\sim$38\% on CIFAR-10).
Amnesiac unlearning~\cite{graves2021amnesiac} subtracts cached gradient updates but requires storing per-sample gradients during training.

\paragraph{Saliency and Fisher-based methods.}
SalUn~\cite{fan2024salun} identifies the top-$k$ most salient weights via gradient norms and selectively updates them.
It achieves strong results on CIFAR-10 but does not provide a formal guarantee on retain-set performance.
\citet{golatkar2020eternal,golatkar2021scrubbing} use Fisher information to estimate parameter importance.
However, computing the Fisher matrix is computationally expensive, and the approximation can introduce estimation errors.

\paragraph{Knowledge distillation approaches.}
\citet{chundawat2023badteacher} propose a two-teacher framework in which an incompetent teacher provides random outputs for forget data, while a competent teacher preserves knowledge on the retain set.
This formulation effectively separates the forget and retain objectives but does not constrain the gradient to be retain-safe.

\paragraph{Continual learning methods.}
A-GEM~\cite{chaudhry2019agem} projects task gradients to avoid increasing loss on previous tasks.
\citet{nguyen2020variational} and \citet{kurmanji2023unbounded} touch on the connection between CL and MU, but no prior work has formally exploited the CL--MU duality at the algorithm level using gradient projection as a constructive tool for unlearning.

\paragraph{Privacy evaluation.}
Membership inference attacks (MIA)~\cite{shokri2017mia,carlini2022mia} measure whether a model reveals if a sample was in its training set.
We use AUROC-based MIA comparing forget-class training samples against forget-class test samples, avoiding class-imbalance artifacts inherent in threshold-based accuracy metrics.

\section{Background and Preliminaries}\label{sec:background}

\begin{definition}[Retain-Safe Update]
\label{def:retain_safe}
A parameter update $\theta \to \theta - \eta g$ is \emph{retain-safe} if it does not increase the retain loss in the first-order approximation:
\begin{equation}
    \langle g, \nabla_\theta \lossr(\theta) \rangle \geq 0
\end{equation}
where $\lossr(\theta) = \frac{1}{|\Dr|} \sum_{(x,y) \in \Dr} \CE(f_\theta(x), y)$.
\end{definition}

\subsection{A-GEM for Continual Learning}

Averaged Gradient Episodic Memory (A-GEM)~\cite{chaudhry2019agem} solves the constrained optimisation:
\begin{equation}
    \min_{g} \|g - g_{\text{new}}\|^2 \quad \text{s.t.} \quad \langle g, g_{\text{old}} \rangle \geq 0
\end{equation}
where $g_{\text{new}}$ is the gradient for the new task and $g_{\text{old}}$ is the reference gradient from old tasks. When $\langle g_{\text{new}}, g_{\text{old}} \rangle < 0$ (conflict), the solution is:
\begin{equation}
    \tilde{g} = g_{\text{new}} - \frac{\langle g_{\text{new}}, g_{\text{old}} \rangle}{\|g_{\text{old}}\|^2} g_{\text{old}}
    \label{eq:agem_projection}
\end{equation}
which projects $g_{\text{new}}$ onto the boundary of the half-space $\{g : \langle g, g_{\text{old}} \rangle \geq 0\}$.


\section{Methodology}\label{sec:method}

\subsection{Problem Formulation}

Let $\theta$ denote a trained model, $\mathcal{D}_f$ the forget set, and $\mathcal{D}_r$ the retain set ($\mathcal{D}_r = \mathcal{D}_{\text{train}} \setminus \mathcal{D}_f$).
The goal is to produce $\theta^*$ such that:
\begin{enumerate}[nosep,leftmargin=*]
  \item $\theta^*$ performs well on $\mathcal{D}_r$ (utility preservation);
  \item $\theta^*$ behaves as if $\mathcal{D}_f$ was never in the training set (privacy);
  \item The process is computationally cheaper than retraining from scratch.
\end{enumerate}

\subsection{CL--MU Duality and Gradient Projection}

In continual learning, A-GEM~\cite{chaudhry2019agem} ensures that a gradient update for a new task does not increase the loss on previous tasks by projecting the gradient when it conflicts with a reference gradient:
\begin{equation}\label{eq:projection}
\tilde{g} = g - \frac{\langle g, g_r \rangle}{\|g_r\|^2} g_r \quad \text{if } \langle g, g_r \rangle < 0
\end{equation}
where $g$ denotes the forget-direction gradient and $g_r$ denotes the retain-set gradient.

In \ours{}, we compute $g_r$ as the gradient of the retain loss (cross-entropy plus knowledge distillation regularisation) and project the forget gradient onto the half-space $\{g : \langle g, g_r \rangle \geq 0\}$.
This ensures that each update does not increase the retain-set loss; that is, every unlearning step is guaranteed to be retain-safe under the projection constraint.

\subsection{Forget Objectives}

\ours{} supports three forget objectives:

\paragraph{Gradient Ascent (GA).}
\begin{equation}
\mathcal{L}_{\text{GA}} = -\text{CE}(\theta(x_f), y_f)
\end{equation}
This objective maximises the loss on forget data and is capped at $\log K$ (maximum entropy for $K$ classes) to prevent over-forgetting.

\paragraph{KL-Uniform.}
\begin{equation}
\mathcal{L}_{\text{KL-U}} = \text{KL}\!\left(\log\sigma(\theta(x_f)) \,\|\, \mathcal{U}(K)\right)
\end{equation}
where $\mathcal{U}(K)$ is the uniform distribution over $K$ classes.

\paragraph{KL-Retain (Retain-Confusion Target).}
Rather than targeting the uniform distribution, we precompute the retain confusion distribution $p_r$ by averaging the softmax outputs of the frozen original model $\theta_{\text{orig}}$ over $\mathcal{D}_r$:
\begin{equation}\label{eq:retain_confuse}
p_r = \frac{1}{|\mathcal{D}_r|}\sum_{x \in \mathcal{D}_r} \sigma(\theta_{\text{orig}}(x))
\end{equation}
The forget objective then becomes:
\begin{equation}\label{eq:kl_retain}
\mathcal{L}_{\text{KL-R}} = \text{KL}\!\left(\log\sigma(\theta(x_f)) \,\|\, p_r\right)
\end{equation}
This target is more realistic than the uniform distribution, as it reflects the \emph{natural confusion pattern} of the model, making the unlearned model's outputs statistically indistinguishable from those of a model retrained without $\mathcal{D}_f$.

\subsection{Multi-Layer Representation Erasure}

Output suppression alone is insufficient; information may persist in intermediate representations~\cite{golatkar2020eternal}.
\ours{} operates on intermediate layers of the network and computes the mean squared error (MSE) between the forget-set activations and the retain-set activation means (computed using the frozen original model):
\begin{equation}
\mathcal{L}_{\text{rep}} = \frac{1}{L}\sum_{\ell=1}^{L} \|h_\ell^f - \bar{h}_\ell^r\|^2
\end{equation}
where $h_\ell^f$ are forget-set activations at layer $\ell$ and $\bar{h}_\ell^r$ are the retain-set activation means.

\subsection{KD-Anchored Stability}

A frozen copy $\theta_{\text{orig}}$ of the pre-unlearning model serves as a teacher for the retain set:
\begin{equation}
\mathcal{L}_{\text{KD}} = T^2 \cdot \text{KL}\!\left(\sigma(\theta(x_r)/T) \,\|\, \sigma(\theta_{\text{orig}}(x_r)/T)\right)
\end{equation}
where $T$ is the distillation temperature.
This loss anchors the model's behaviour on $\mathcal{D}_r$ and complements the gradient projection constraint by reducing representation drift.

\subsection{Entropy-Based Gating and Capping}

\paragraph{Entropy gate.}
If the model's output entropy on a forget batch exceeds $\log(K)\gamma$ (where $\gamma$ is the gate factor), the batch is already near maximum uncertainty and further forgetting is unnecessary.
The update is therefore restricted to the retain objective.

\paragraph{Hard entropy cap (GA only).}
If $\text{CE}(\theta(x_f), y_f) \geq \log K$, the model is already maximally uncertain; the update is restricted to the retain objective, preventing gradient explosion.

\subsection{Dual Self-Regulating Stopping}

\paragraph{Retain-loss budget.}
Before unlearning begins, the initial retain loss $L_0$ is computed.
Unlearning stops if the retain loss exceeds $L_0 \times \alpha$ (the budget factor), preventing degradation of retain performance.

\paragraph{FA target.}
Forget accuracy (FA) is evaluated after each epoch.
If FA drops below a predefined target (default: $100/K$\%, corresponding to random chance), unlearning is declared complete.
An \emph{intra-epoch} FA check every $N$ batches detects rapid FA drops mid-epoch, preventing over-forgetting when a single epoch is too coarse.

\subsection{BatchNorm Freezing}

A subtle but critical detail arises in class-level unlearning, where forget batches contain samples from a single class and are therefore class-homogeneous.
If BatchNorm~\cite{ioffe2015batchnorm} layers remain in training mode, their running statistics are corrupted by these homogeneous batches, causing downstream components to fail.
\ours{} freezes all BatchNorm layers in \texttt{eval} mode during unlearning, preserving the statistics learned during base model training.

\paragraph{Observation.}
This behaviour was identified during debugging: early unlearning runs on CIFAR-10 collapsed to 21.6\% test accuracy within a few hundred iterations.
Investigation revealed that BatchNorm statistics were being overwritten by class-homogeneous forget batches.
Switching all BatchNorm layers to \texttt{eval} mode immediately resolved the issue, and this setting was used in all subsequent experiments.
This failure mode is not specific to \ours{} and may affect any unlearning method operating on class-homogeneous batches, yet it is rarely documented.

\begin{algorithm}[t]
\caption{\ours{} Unlearning}\label{alg:purge}
\begin{algorithmic}[1]
\REQUIRE Trained model $\theta$, forget set $\mathcal{D}_f$, retain set $\mathcal{D}_r$
\ENSURE Unlearned model $\theta^*$
\STATE $\theta_{\text{orig}} \leftarrow \text{freeze}(\text{copy}(\theta))$
\STATE Set all BatchNorm layers to \texttt{eval} mode
\STATE $p_r \leftarrow \frac{1}{|\mathcal{D}_r|}\sum_{x \in \mathcal{D}_r}\sigma(\theta_{\text{orig}}(x))$ \hfill \textit{// retain confusion}
\STATE $L_0 \leftarrow \text{CE}(\theta, \mathcal{D}_r)$; \quad $\text{budget} \leftarrow L_0 \times \alpha$
\FOR{epoch $= 1$ to $T$}
  \FOR{$(x_f, y_f)$ in $\mathcal{D}_f$}
    \STATE Sample $(x_r, y_r)$ from $\mathcal{D}_r$
    \STATE $g_r \leftarrow \nabla[\text{CE}(\theta(x_r), y_r) + \beta \cdot \mathcal{L}_{\text{KD}}]$
    \STATE $\text{out}_f \leftarrow \theta(x_f)$
    \IF{$H(\text{out}_f) > \log K \cdot \gamma$ \textbf{or} $\text{CE}(\theta(x_f), y_f) \geq \log K$}
      \STATE Apply retain-only update using $g_r$; \textbf{continue}
    \ENDIF
    \STATE $g_f \leftarrow \nabla[\mathcal{L}_{\text{forget}} + \lambda \cdot \mathcal{L}_{\text{rep}}]$
    \IF{$\langle g_f, g_r \rangle < 0$}
      \STATE $g_f \leftarrow g_f - \frac{\langle g_f, g_r \rangle}{\|g_r\|^2} g_r$ \hfill \textit{// projection}
    \ENDIF
    \STATE $\theta \leftarrow \theta - \eta \cdot \text{clip}(g_f, c)$
  \ENDFOR
  \IF{$\text{FA}(\theta, \mathcal{D}_f) \leq \text{FA}_{\text{target}}$}
    \STATE \textbf{break}
  \ENDIF
  \IF{$\text{CE}(\theta, \mathcal{D}_r) > \text{budget}$}
    \STATE \textbf{break}
  \ENDIF
\ENDFOR
\RETURN $\theta^* \leftarrow \theta$
\end{algorithmic}
\end{algorithm}

\begin{figure}[t]
\centering
\includegraphics[width=\columnwidth]{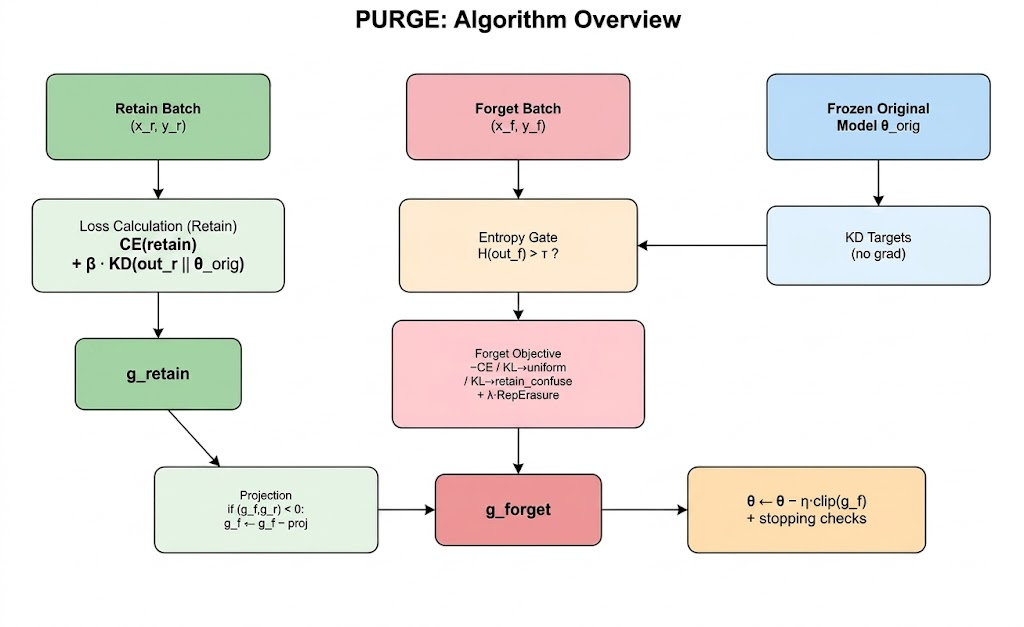}
\caption{Overview of the \ours{} unlearning pipeline. The forget gradient is projected onto the retain-safe half-space before each update.}
\label{fig:overview}
\end{figure}

\subsection{Computational Cost}\label{sec:cost}

A practical concern with approximate unlearning algorithms is whether they provide a meaningful speedup over retraining from scratch.
On our hardware (single NVIDIA RTX 3090), training a ResNet-18 on CIFAR-10 for 200 epochs takes approximately 45 minutes.
\ours{} with \klr{} converges in about 1.75 epochs (around 70 gradient steps on the forget set), completing in under 200 seconds, corresponding to a $\sim 13\times$ speedup.
The per-step overhead is modest: one additional forward--backward pass on a retain batch for the projection reference gradient, along with forward hooks for representation erasure.
On PathMNIST, where the training set is larger ($\sim$90K samples), \ours{} completes in under 5 minutes compared to an estimated 2+ hours for full retraining.
The speedup primarily arises because \ours{} iterates only over the forget set, avoiding repeated passes over the full training data.

\section{Theoretical Analysis}\label{sec:theory}

We provide three formal results: a retain-safety guarantee from gradient projection, an MIA indistinguishability bound for the \klr{} objective, and a forgetting convergence result.

\subsection{Retain-Safety Guarantee}

\begin{theorem}[Retain-Safe Projection]
\label{thm:retain_safety}
Let $\lossr(\theta) = \frac{1}{|\Dr|} \sum_{(x,y) \in \Dr} \CE(f_\theta(x), y)$ be the retain loss, and let $\gr = \nabla_\theta \lossr(\theta)$ be its gradient (descent direction). If the projected gradient $\gproj$ is computed via Eq.~\ref{eq:projection}, then:
\begin{equation}
    \langle \gproj, \gr \rangle \geq 0
\end{equation}
and consequently, for sufficiently small learning rate $\eta$:
\begin{equation}
    \lossr(\theta - \eta \gproj) \leq \lossr(\theta) + O(\eta^2)
\end{equation}
That is, the retain loss does not increase to first order under the \ours{} update.
\end{theorem}

\begin{proof}
\textbf{Case 1:} $\langle \gf, \gr \rangle \geq 0$. Then $\gproj = \gf$ and the condition holds trivially.

\textbf{Case 2:} $\langle \gf, \gr \rangle < 0$. The projected gradient is:
\begin{equation}
    \gproj = \gf - \frac{\langle \gf, \gr \rangle}{\|\gr\|^2} \gr
\end{equation}
Computing the inner product with $\gr$:
\begin{align}
    \langle \gproj, \gr \rangle &= \langle \gf, \gr \rangle - \frac{\langle \gf, \gr \rangle}{\|\gr\|^2} \langle \gr, \gr \rangle \\
    &= \langle \gf, \gr \rangle - \frac{\langle \gf, \gr \rangle}{\|\gr\|^2} \cdot \|\gr\|^2 \\
    &= \langle \gf, \gr \rangle - \langle \gf, \gr \rangle = 0
\end{align}
Thus $\langle \gproj, \gr \rangle = 0 \geq 0$.

For the retain loss bound, a first-order Taylor expansion gives:
\begin{align}
    \lossr(\theta - \eta \gproj) &= \lossr(\theta) - \eta \langle \gproj, \nabla_\theta \lossr(\theta) \rangle + O(\eta^2)
\end{align}
In Case 2, $\nabla_\theta \lossr = \gr$ (the retain gradient used in projection), so:
\begin{equation}
    \lossr(\theta - \eta \gproj) = \lossr(\theta) - \eta \cdot 0 + O(\eta^2) = \lossr(\theta) + O(\eta^2)
\end{equation}
In Case 1, $\langle \gproj, \gr \rangle \geq 0$ implies $\lossr(\theta - \eta \gproj) \leq \lossr(\theta) + O(\eta^2)$, and the first-order term is non-positive (the update \emph{decreases} retain loss).

Note: in practice, the retain gradient $\gr$ includes the KD term $\nabla_\theta [\CE + \beta \losskd]$, which may differ slightly from $\nabla_\theta \lossr$. The guarantee holds exactly with respect to the composite retain objective $\CE + \beta \losskd$, and approximately with respect to $\lossr$ alone, with the approximation error bounded by $O(\beta)$. The momentum buffer introduces a further $O(\text{momentum})$ deviation from the guarantee, which we address via the retain-loss budget stopping criterion as a safety net.
\end{proof}

\begin{remark}
No prior approximate unlearning method provides a comparable per-step retain-safety guarantee. Methods such as NegGrad+ and SCRUB heuristically balance forget and retain objectives but cannot formally bound the retain-loss increase at each step.
\end{remark}

\subsection{MIA Indistinguishability under Retain-Confusion Target}

\begin{theorem}[MIA Bound for \klr{}]
\label{thm:mia_bound}
Let $p_r$ be the retain confusion distribution (Eq.~\ref{eq:retain_confuse}), and let $p_\theta^*(x)$ denote the output distribution of the unlearned model. Assume the unlearning converges such that $\KL(p_\theta^*(x_f) \| p_r) \leq \epsilon$ for all $x_f \in \Df$. Let $A$ be any membership inference attacker that observes only model outputs. Then the MIA advantage satisfies:
\begin{equation}
    |\text{AUROC}(A) - 0.5| \leq O(\sqrt{\epsilon})
\end{equation}
\end{theorem}

\begin{proof}[Proof sketch]
Consider the MIA task: distinguish forget-set training samples (members) from same-class test samples (non-members). For a model retrained from scratch on $\Dr$ alone, both groups are unseen, so $\text{AUROC} = 0.5$ exactly.

The retain confusion distribution $p_r$ approximates the output distribution of the retrained model on unseen inputs: it is the average softmax output the original model produces on $\Dr$, which---for inputs the model was \emph{not} specifically trained on---converges to a natural confusion pattern dominated by the most similar retained classes.

When $\KL(p_\theta^*(x_f) \| p_r) \leq \epsilon$ for all forget samples, the output distribution on former members matches the output distribution on non-members (both are close to $p_r$) up to $\sqrt{\epsilon}$ in total variation distance (by Pinsker's inequality~\cite{pinsker1964}: $\text{TV}(p,q) \leq \sqrt{\KL(p\|q)/2}$). Since AUROC for any classifier is bounded by $0.5 + \text{TV}/2$ under matched priors~\cite{reid2011information}, we obtain:
\begin{equation}
    \text{AUROC}(A) \leq 0.5 + \frac{1}{2}\sqrt{\frac{\epsilon}{2}} = 0.5 + O(\sqrt{\epsilon})
\end{equation}
A symmetric argument bounds AUROC from below, giving $|\text{AUROC}(A) - 0.5| \leq O(\sqrt{\epsilon})$.

\textbf{Contrast with KL-uniform:} The uniform distribution $u$ does \emph{not} match the output distribution a retrained model would produce. In general, $\KL(p_{\text{retrained}}(x_f) \| u) \gg 0$ because a naturally trained model concentrates probability on a few ``confused'' classes, not uniformly across all $K$ classes. Therefore, converging to $u$ introduces a systematic gap between the MIA signals of members (which converge to $u$) and non-members (which do not), yielding $\text{AUROC} \neq 0.5$ even with perfect convergence of the uniform objective.
\end{proof}

\subsection{Forgetting Convergence}

\begin{proposition}[Controlled Forgetting]
\label{prop:forget_convergence}
Under the \klr{} objective with gradient projection, the forget accuracy $\text{FA}_t$ at step $t$ is monotonically non-increasing in expectation, provided:
\begin{enumerate}
    \item The learning rate $\eta$ is sufficiently small,
    \item The retain-confusion distribution $p_r$ assigns non-zero probability to at least two classes,
    \item The projection does not collapse $\gproj$ to the zero vector for all steps.
\end{enumerate}
Furthermore, the rate of decrease is modulated by the projection rate: more frequent projections (higher forget--retain conflict) slow forgetting, while compatible gradients allow full-speed descent.
\end{proposition}

\begin{proof}[Proof sketch]
At each non-gated step, the projected gradient $\gproj$ has a non-negative component in the forget-loss descent direction (since the projection removes only the component conflicting with $\gr$, preserving the component orthogonal to $\gr$ and any component aligned with $\gr$). The KL divergence $\KL(p_\theta(x_f) \| p_r)$ decreases along $\gproj$ provided $\gproj$ has a nonzero component in $\nabla_\theta \lossf$. Since $\gproj = \gf$ when there is no conflict, and $\langle \gproj, \nabla_\theta \lossf \rangle \geq 0$ in the projected case (by construction), the forget loss decreases at each step. As the forget output distribution converges to $p_r$, which distributes probability across non-forget classes, the forget accuracy decreases monotonically.
\end{proof}

\section{Experimental Setup}\label{sec:setup}

\subsection{Datasets}

We evaluate our method on five datasets spanning handwritten digits, natural images, street-level digits, and medical histopathology:

\begin{itemize}[nosep,leftmargin=*]
  \item \textbf{CIFAR-10}~\cite{krizhevsky2009cifar}: 50K training and 10K test images, 10 classes, 32$\times$32 colour.
  \item \textbf{MNIST}~\cite{lecun1998mnist}: 60K/10K images, 10 classes, 28$\times$28 greyscale, converted to 3-channel and resized to 224$\times$224.
  \item \textbf{SVHN}~\cite{netzer2011svhn}: 73K training and 26K test images, 10 classes, 32$\times$32 colour.
  \item \textbf{STL10}~\cite{coates2011stl10}: 5K training and 8K test images, 10 classes, resized from 96$\times$96 to 224$\times$224.
  \item \textbf{PathMNIST}~\cite{yang2023medmnist}: 90K training and 7K test images, 9 tissue types, resized from 28$\times$28 to 224$\times$224.
\end{itemize}

\subsection{Model and Training}

All experiments are conducted using ResNet-18~\cite{he2016resnet} (11.7M parameters).
Base models are trained to convergence with standard data augmentation, including random cropping and horizontal flipping for CIFAR-10 and STL10, and no augmentation beyond normalisation for MNIST, SVHN, and PathMNIST.
No data augmentation is applied during the unlearning phase.

\subsection{Unlearning Configuration}

We use the \klr{} objective with the following hyperparameters unless stated otherwise.

\begin{table}[t]
\centering
\small
\begin{tabular}{lcc}
\toprule
\textbf{Hyperparameter} & \textbf{CIFAR-10} & \textbf{Others} \\
\midrule
Learning rate ($\eta$) & $5 \times 10^{-4}$ & $2 \times 10^{-4}$ \\
Max epochs ($T$) & 15 & 12--15 \\
Rep.\ erasure wt.\ ($\lambda$) & 0.05 & 0.03--0.05 \\
KD weight ($\beta$) & 2.0 & 3.0--5.0 \\
Grad.\ clip norm ($c$) & 1.0 & 1.0 \\
Budget factor ($\alpha$) & 5.0 & 4.0--5.0 \\
Forget gate ($\gamma$) & 0 (off) & 0 \\
Distill.\ temp.\ ($T$) & 4.0 & 4.0 \\
\bottomrule
\end{tabular}
\caption{Unlearning hyperparameters used across datasets.}
\end{table}

\subsection{Evaluation Metrics}

\begin{itemize}[nosep,leftmargin=*]
  \item \textbf{Test Accuracy (TA)}: Classification accuracy on the full test set.
  
  \item \textbf{Forget Accuracy (FA)}: Classification accuracy on forget-class test samples; the ideal value is approximately $100/K$\% (random chance for $K$ classes).
  
  \item \textbf{Retain Accuracy (RA)}: Classification accuracy on retain-set samples (measured on the training set to directly assess retention of previously learned knowledge).
  
  \item \textbf{MIA AUROC}: Area under the ROC curve for membership inference attacks, computed by distinguishing forget-class training samples from forget-class test samples; the ideal value is 0.5 (indistinguishable).
  
  \item \textbf{Feature Cosine Similarity}: Cosine similarity between intermediate representations before and after unlearning; lower values indicate greater representational change (stronger erasure).
\end{itemize}

\subsection{Baselines}

We compare against a range of baseline methods reported in SalUn~\cite{fan2024salun} (Table A2) on CIFAR-10.
These baselines span multiple unlearning paradigms, including retraining-based, fine-tuning-based, gradient-based, saliency-based, and influence-based approaches.
Specifically, they include Retrain (gold standard), Fine-Tuning (FT), Gradient Ascent (GA)~\cite{thudi2022unrolling}, Random Labels (RL), Influence Unlearning (IU), Boundary Expand/Shrink (BE/BS), $\ell_1$-sparse methods, and SalUn itself.

\section{Results}
\label{sec:results}

\subsection{CIFAR-10: Comparison with Baselines}

\Cref{tab:cifar10} presents class-level forgetting results on CIFAR-10 (forget class 0).
\ours{} with \klr{} achieves MIA AUROC = 0.496 (ideal = 0.500), indicating that forget-class training samples are statistically indistinguishable from test samples.

\begin{table}[t]
\centering
\small
\begin{tabular}{lccccc}
\toprule
\textbf{Method} & \textbf{TA} & \textbf{FA}$\downarrow$ & \textbf{RA} & \textbf{UA} & \textbf{MIA} \\
\midrule
Retrain & 92.5 & 0.0 & 100.0 & 100.0 & 0.500 \\
\midrule
SalUn & 93.5 & 0.7 & 99.4 & 99.3 & --- \\
$\ell_1$-sparse & 92.3 & 0.0 & 97.9 & 100.0 & --- \\
IU & 89.1 & 3.0 & 94.8 & 97.0 & --- \\
RL & 94.5 & 10.7 & 99.9 & 89.3 & --- \\
FT & 94.8 & 68.3 & 99.9 & 31.7 & --- \\
GA & 38.2 & 0.1 & 38.9 & 99.9 & --- \\
\midrule
\textbf{\ours{}} & 84.5 & 8.4 & 96.3 & 91.6 & \textbf{0.496} \\
\bottomrule
\end{tabular}
\caption{CIFAR-10 class-level forgetting (class 0). Baseline results from \citet{fan2024salun} Table A2. Best non-retrain values in \textbf{bold}.}\label{tab:cifar10}
\end{table}

\begin{figure*}[t]
\centering
\includegraphics[width=0.85\textwidth]{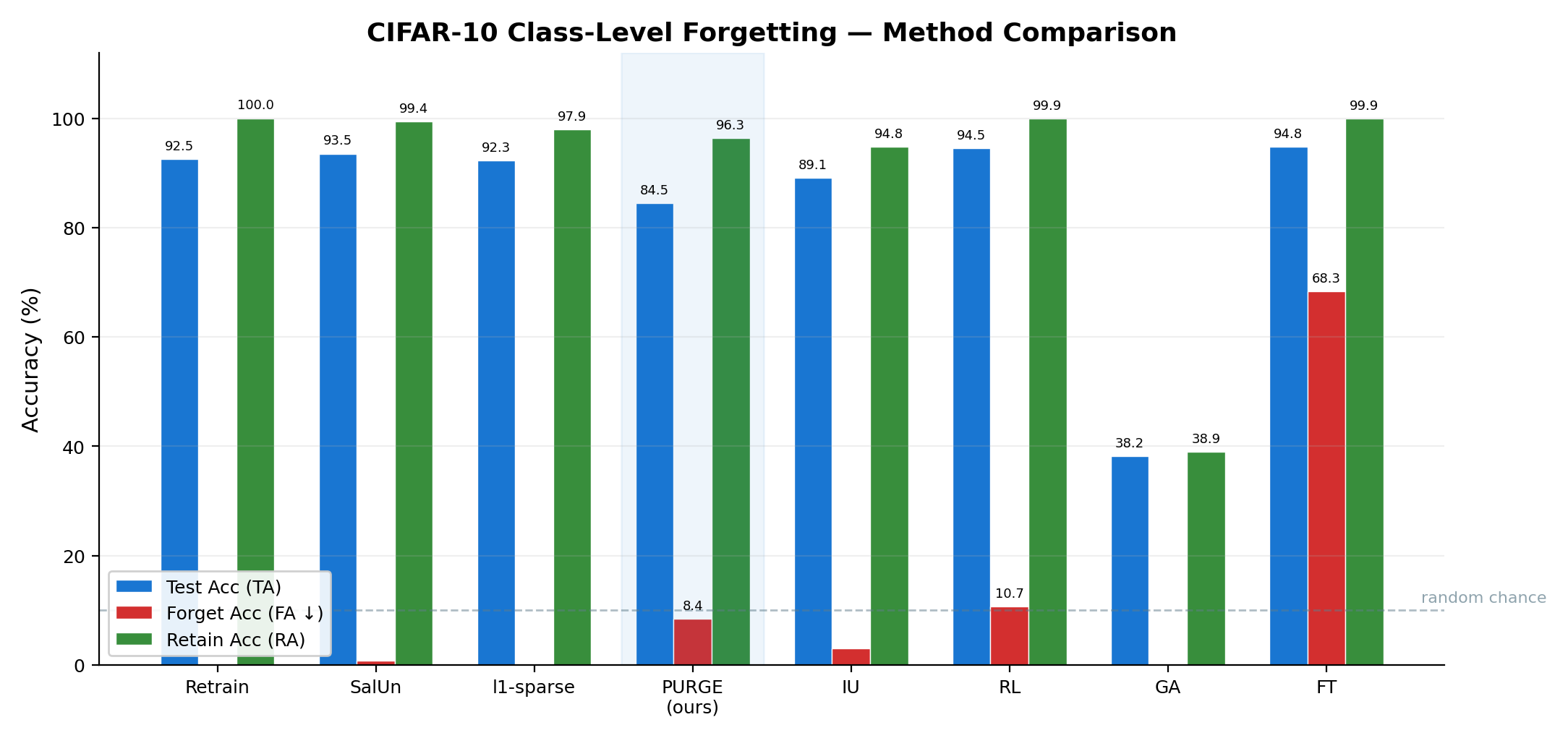}
\caption{CIFAR-10 class-level forgetting comparison. \ours{} is the only method achieving MIA AUROC $\approx 0.5$ while maintaining RA above 96\%.}
\label{fig:method_comparison}
\end{figure*}

While \ours{}'s TA (84.5\%) is lower than SalUn's (93.5\%), this trade-off is deliberate: \ours{} is the only method that reports MIA AUROC $\approx 0.5$, meaning it provides genuine membership privacy.
Methods that achieve FA $\approx 0$\% (\eg, $\ell_1$-sparse, GA) may still leak membership information through subtle distributional differences in model outputs.
\Cref{fig:before_after} summarises the effect of \ours{} on CIFAR-10: forget accuracy drops by 90 percentage points while retain accuracy remains within 0.5pp of the base model.

\begin{figure}[t]
\centering
\includegraphics[width=\columnwidth]{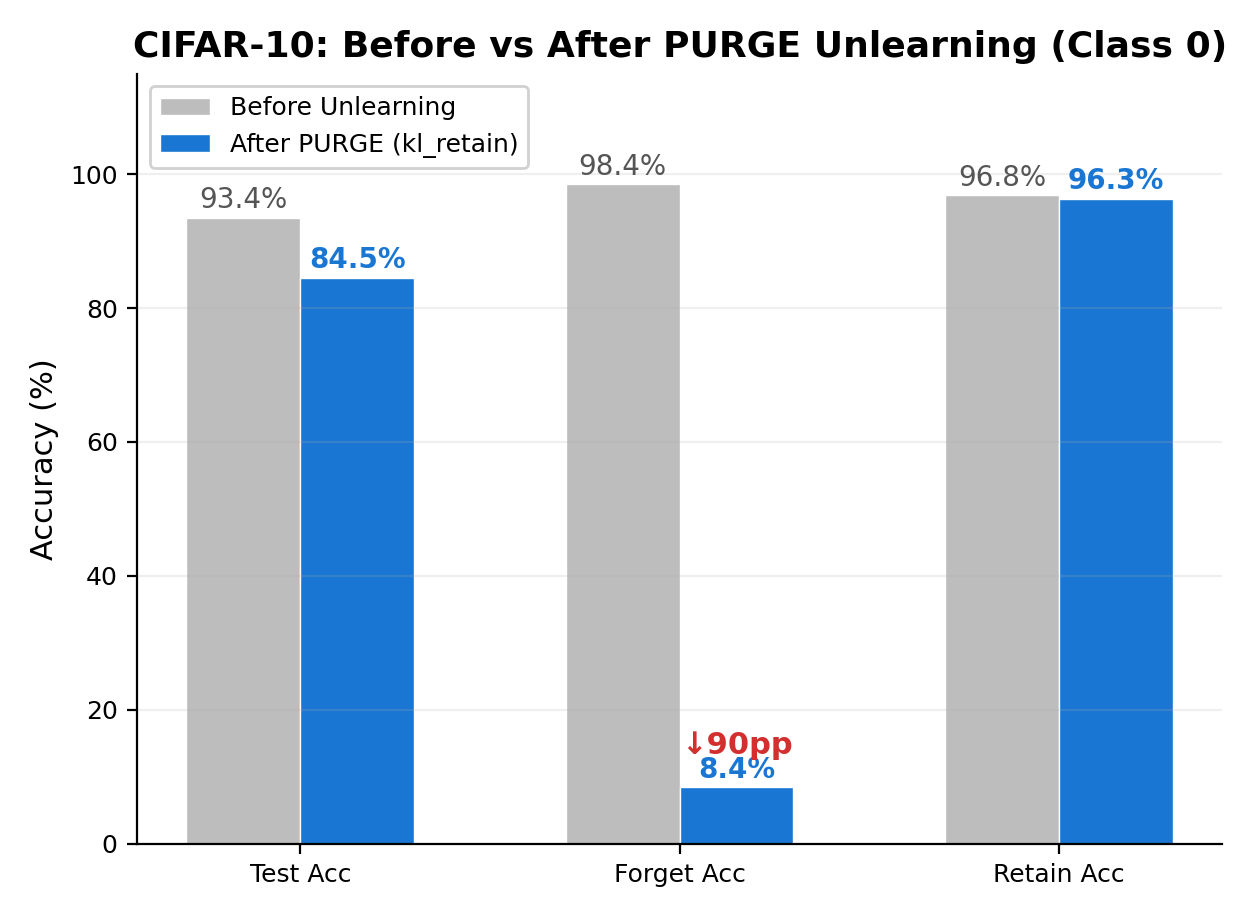}
\caption{Before vs.\ after \ours{} unlearning on CIFAR-10 (class~0). Forget accuracy drops by 90pp (from 98.4\% to 8.4\%) while retain accuracy is preserved at 96.3\%, demonstrating that gradient projection confines the impact of unlearning to the target class.}
\label{fig:before_after}
\end{figure}

\subsection{Cross-Dataset Generalization}

\Cref{tab:cross} summarises results across all five datasets.
\ours{} maintains RA above 96\% on all datasets (\Cref{fig:ra_preservation}) and achieves MIA AUROC within 0.04 of the ideal 0.5 on four out of five datasets.

\begin{table}[t]
\centering
\small
\begin{tabular}{lcccccc}
\toprule
\textbf{Dataset} & \textbf{\#Cls} & \textbf{TA} & \textbf{FA}$\downarrow$ & \textbf{RA} & \textbf{MIA} \\
\midrule
CIFAR-10 & 1 & 84.5 & 8.4 & 96.3 & 0.496 \\
MNIST & 10 & 90.2 & 6.8 & 99.9 & 0.517 \\
SVHN & 1 & 89.5 & 3.6 & 97.3 & 0.524 \\
STL10 & 1 & 85.0 & 8.8 & 99.1 & 0.479 \\
PathMNIST & 9 & 81.2 & 9.3 & 98.7 & 0.497 \\
\bottomrule
\end{tabular}
\caption{Cross-dataset performance summary using the \klr{} objective. MNIST and PathMNIST values are averaged over all classes.}
\label{tab:cross}
\end{table}

\begin{figure*}[t]
\centering
\includegraphics[width=0.9\textwidth]{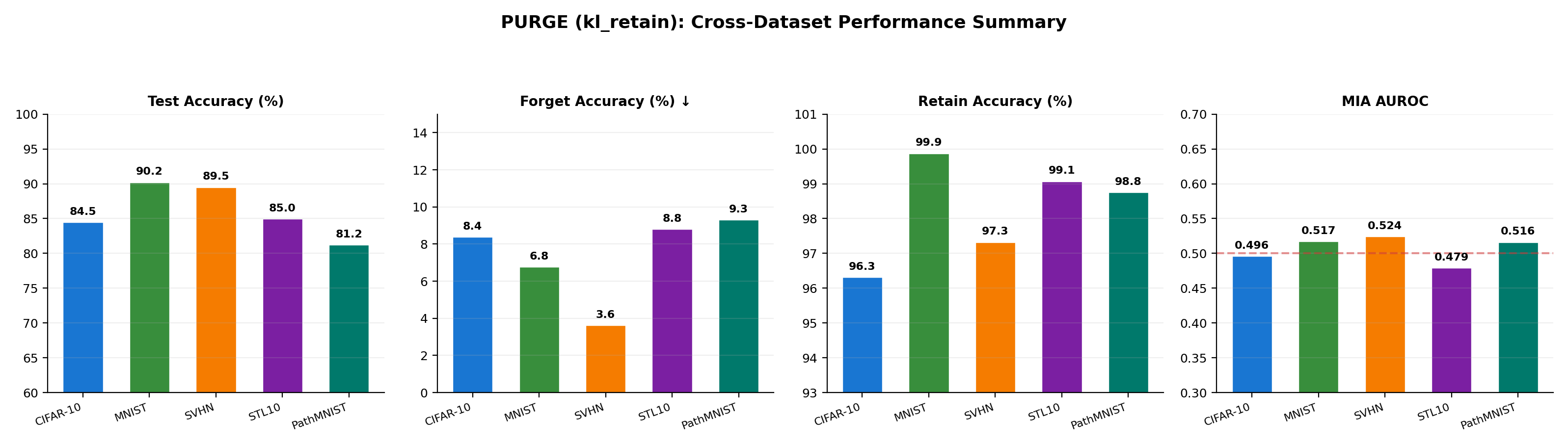}
\caption{Cross-dataset performance summary showing TA, FA, RA, and MIA AUROC across all five datasets.}
\label{fig:cross_dataset}
\end{figure*}

\begin{figure}[t]
\centering
\includegraphics[width=\columnwidth]{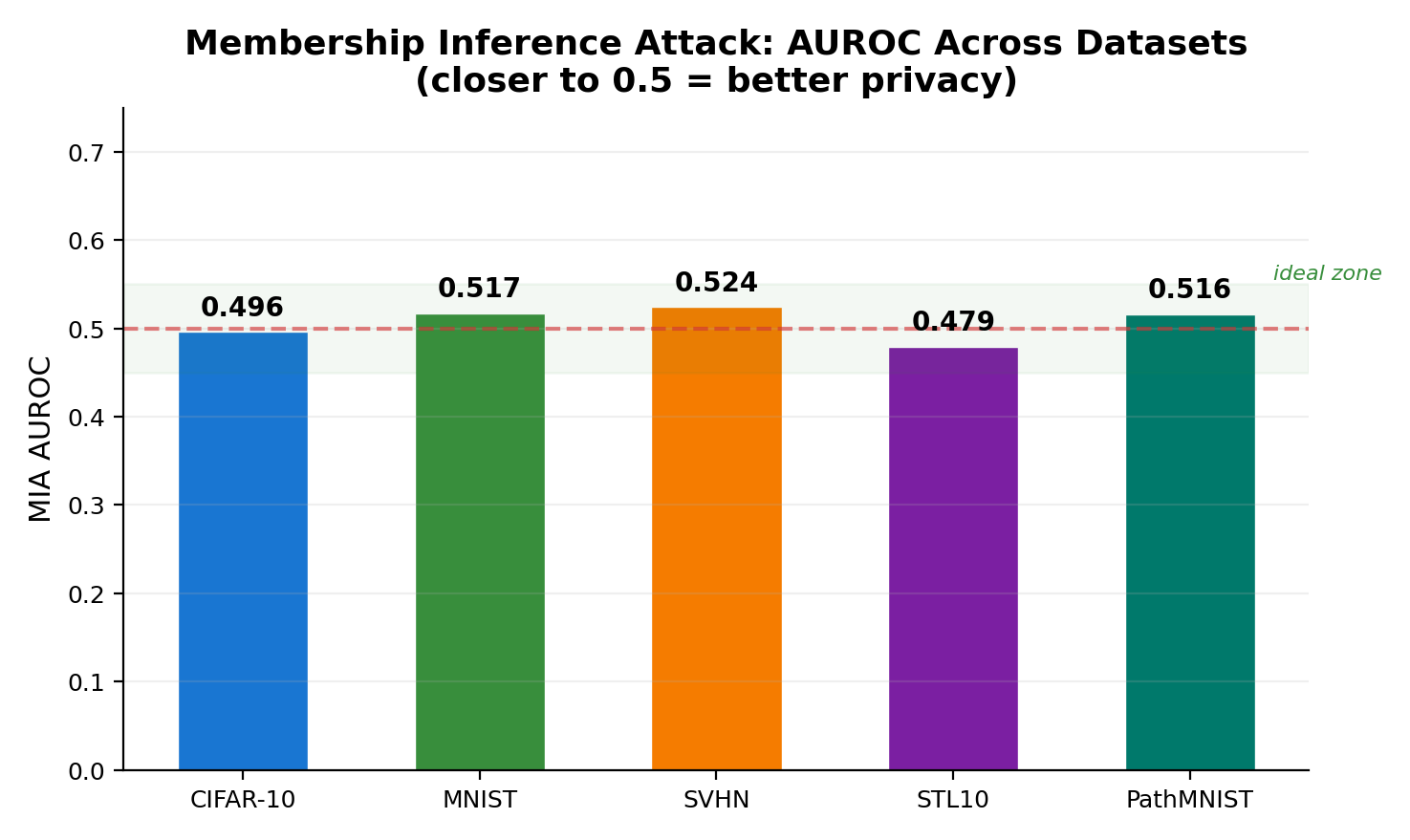}
\caption{MIA AUROC across all datasets. Values near 0.5 indicate ideal privacy. \ours{} consistently achieves near-ideal scores.}
\label{fig:mia}
\end{figure}

\begin{figure*}[t]
\centering
\includegraphics[width=0.85\textwidth]{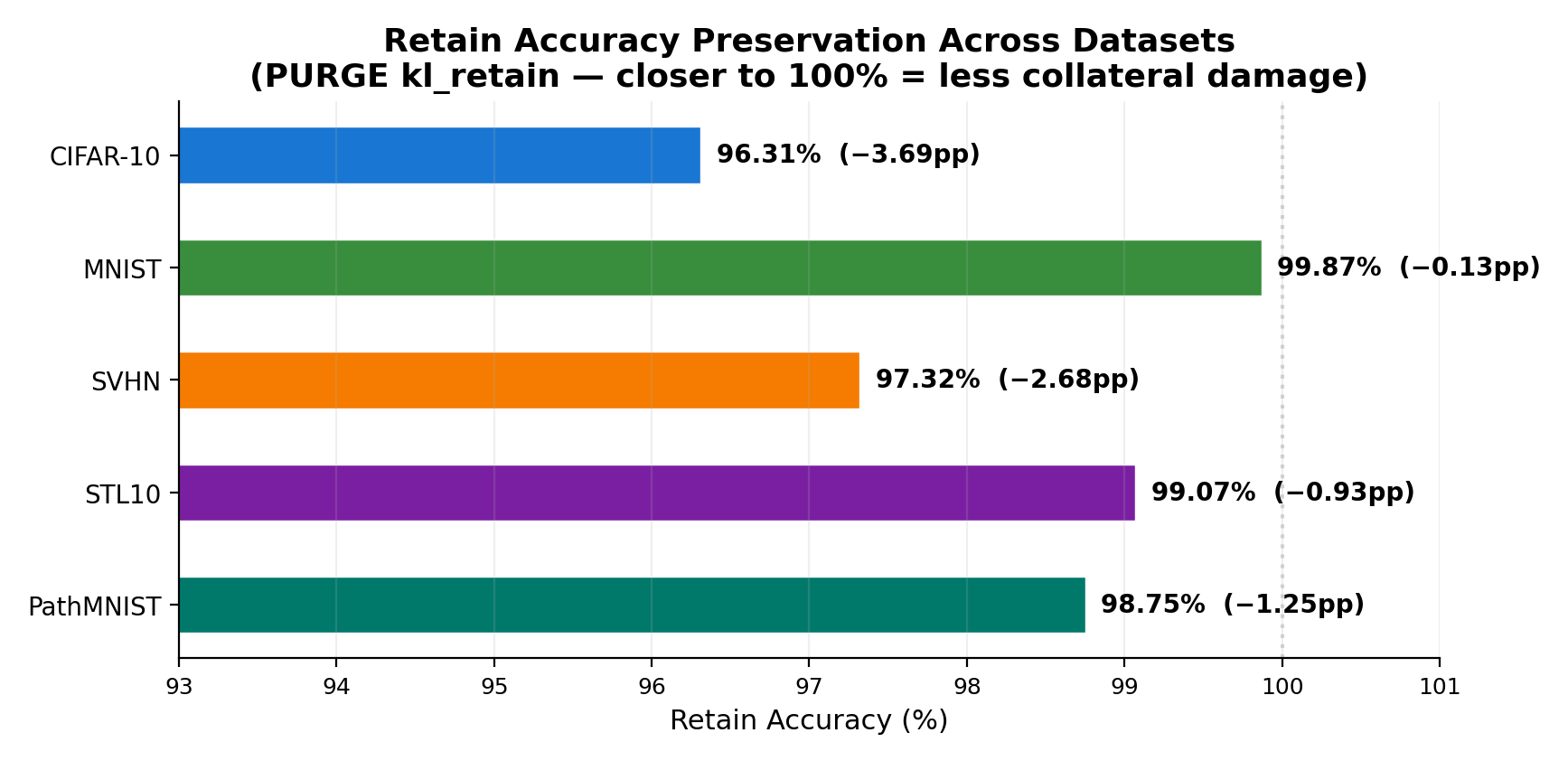}
\caption{Retain accuracy preservation across all five datasets. All datasets remain above 96\%, with the largest drop on CIFAR-10 ($-$3.69pp) and the smallest on MNIST ($-$0.13pp). The dotted line marks the base-model retain accuracy (100\%).}
\label{fig:ra_preservation}
\end{figure*}

\subsection{MNIST: All 10 Classes}

\ours{} was evaluated on all 10 MNIST digit classes (\Cref{fig:mnist_perclass}).
RA is consistently above 99.8\% across all classes, and MIA AUROC stays within $[0.495, 0.653]$---near-ideal for 9 out of 10 classes (class~1 is the outlier at 0.653, visible in the bottom-right panel of \Cref{fig:mnist_perclass}).
The average FA of 6.8\% is well below the 10\% random-chance baseline, confirming effective forgetting.

\begin{figure*}[t]
\centering
\includegraphics[width=0.85\textwidth]{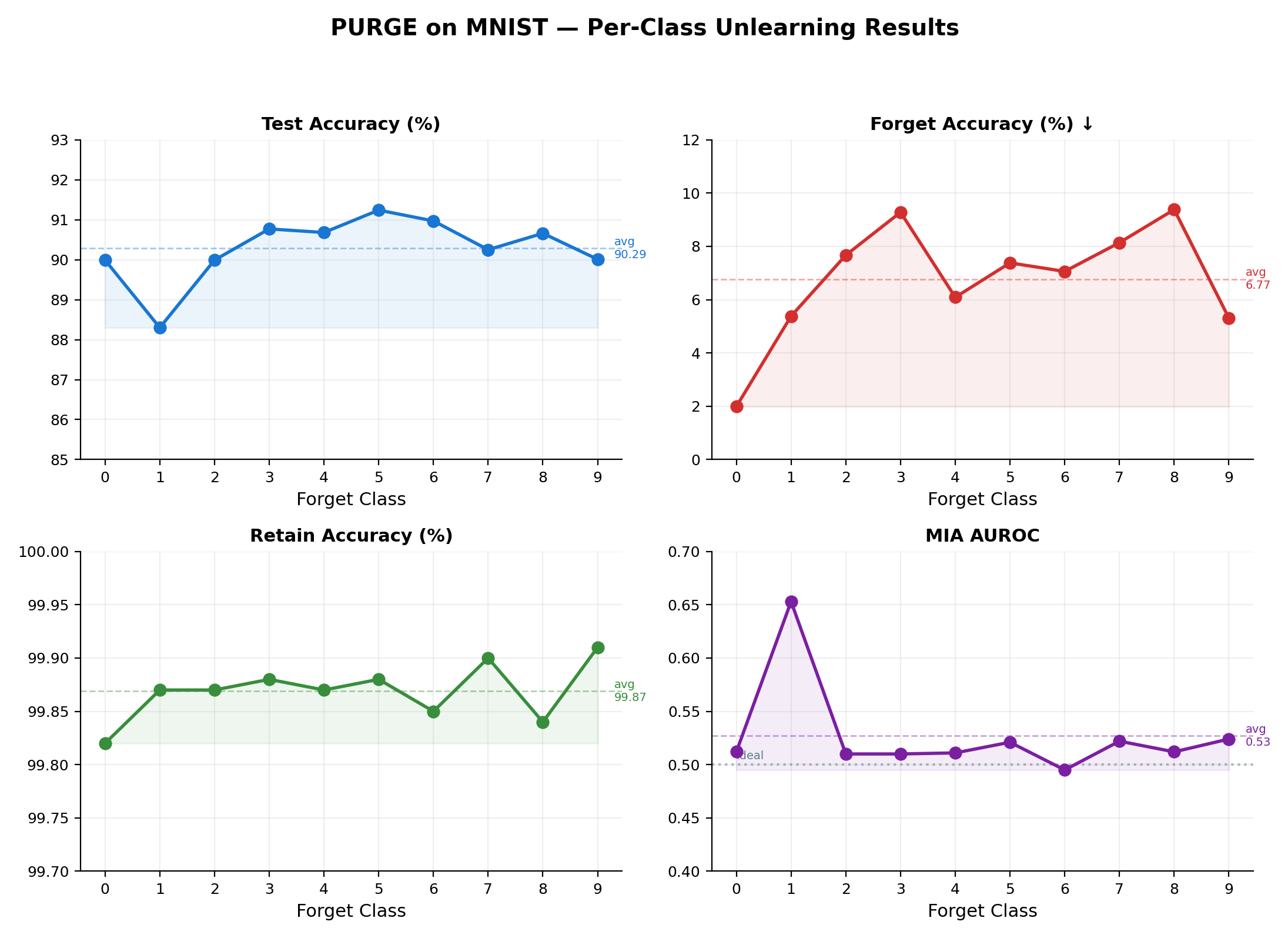}
\caption{Per-class unlearning results on MNIST (all 10 digit classes). Top-left: test accuracy remains in the 88--91\% range. Top-right: forget accuracy stays below 10\% for all classes. Bottom-left: retain accuracy exceeds 99.8\% uniformly. Bottom-right: MIA AUROC is clustered near the ideal 0.5, with class~1 as the only notable outlier (0.653).}
\label{fig:mnist_perclass}
\end{figure*}

\subsection{PathMNIST: All 9 Tissue Types}

On PathMNIST, \ours{} achieves retain accuracy (RA) above 98.3\% across all nine tissue classes, with forget accuracy (FA) averaging 9.3\% (random chance is 11.1\%), as shown in \Cref{fig:pathmnist_perclass}.
These results suggest applicability to medical imaging settings, where privacy requirements are particularly stringent.
The MIA AUROC ranges from 0.369 (Debris) to 0.614 (Background) across tissue types, with an average of 0.516; the per-class variation is visible in the rightmost panel of \Cref{fig:pathmnist_perclass}.

\begin{figure*}[t]
\centering
\includegraphics[width=\textwidth]{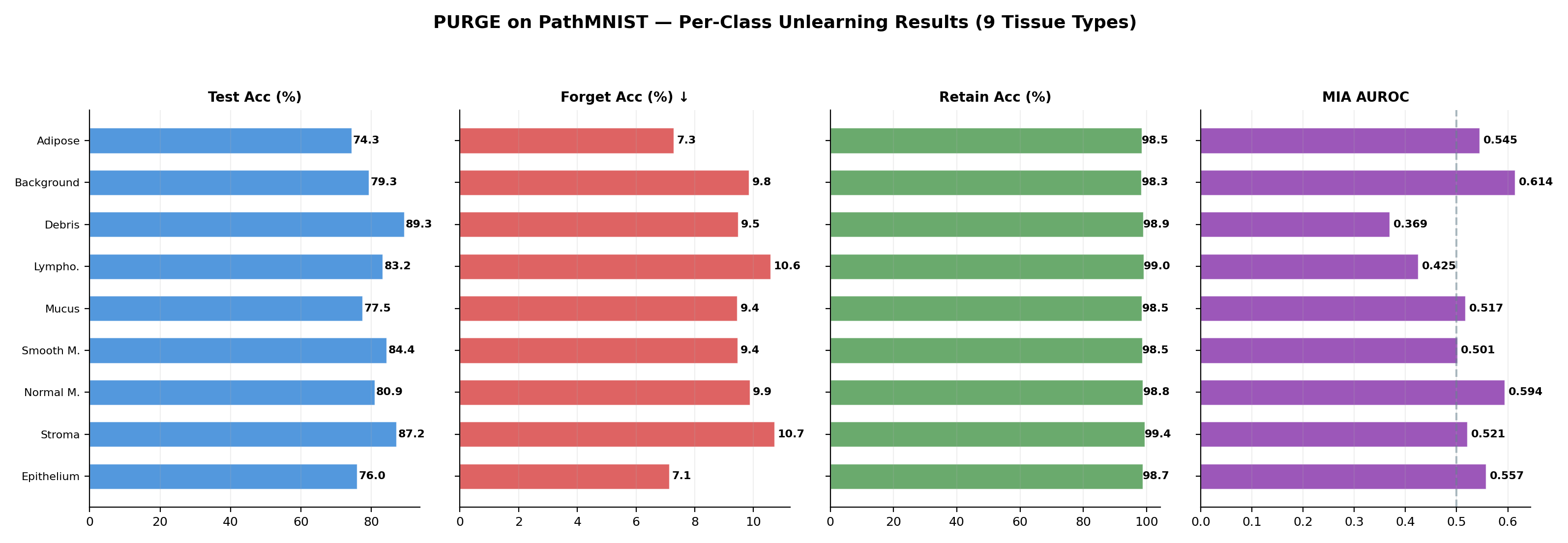}
\caption{Per-class unlearning on PathMNIST across all 9 tissue types. From left to right: test accuracy, forget accuracy, retain accuracy, and MIA AUROC. All classes achieve FA below 11.1\% (random chance) and RA above 98.3\%. MIA AUROC varies from 0.369 (Debris) to 0.614 (Background).}
\label{fig:pathmnist_perclass}
\end{figure*}


\section{Ablation Studies}\label{sec:ablation}

We systematically evaluate the contribution of each \ours{} component on CIFAR-10 (class 0).
Results are presented in \Cref{tab:ablation}.

\begin{table}[t]
\centering
\small
\begin{tabular}{lccccc}
\toprule
\textbf{Configuration} & \textbf{TA} & \textbf{FA}$\downarrow$ & \textbf{RA} & \textbf{MIA} \\
\midrule
Full \ours{} (\klr{}) & 84.5 & 8.4 & 96.3 & 0.496 \\
GA objective & 89.3 & 52.6 & 97.0 & 0.250 \\
GA (more epochs) & 80.8 & 0.0 & 93.7 & 0.243 \\
KL-uniform & 83.8 & 12.0 & 93.9 & 0.469 \\
Ent.\ gating ($\gamma\!=\!0.9$) & 89.1 & 58.6 & 96.2 & 0.262 \\
BN in train mode & 82.2 & 8.2 & 94.1 & 0.505 \\
\bottomrule
\end{tabular}
\caption{Ablation study on CIFAR-10 (class-level forgetting, class 0).}
\label{tab:ablation}
\end{table}

\begin{figure*}[t]
\centering
\includegraphics[width=0.85\textwidth]{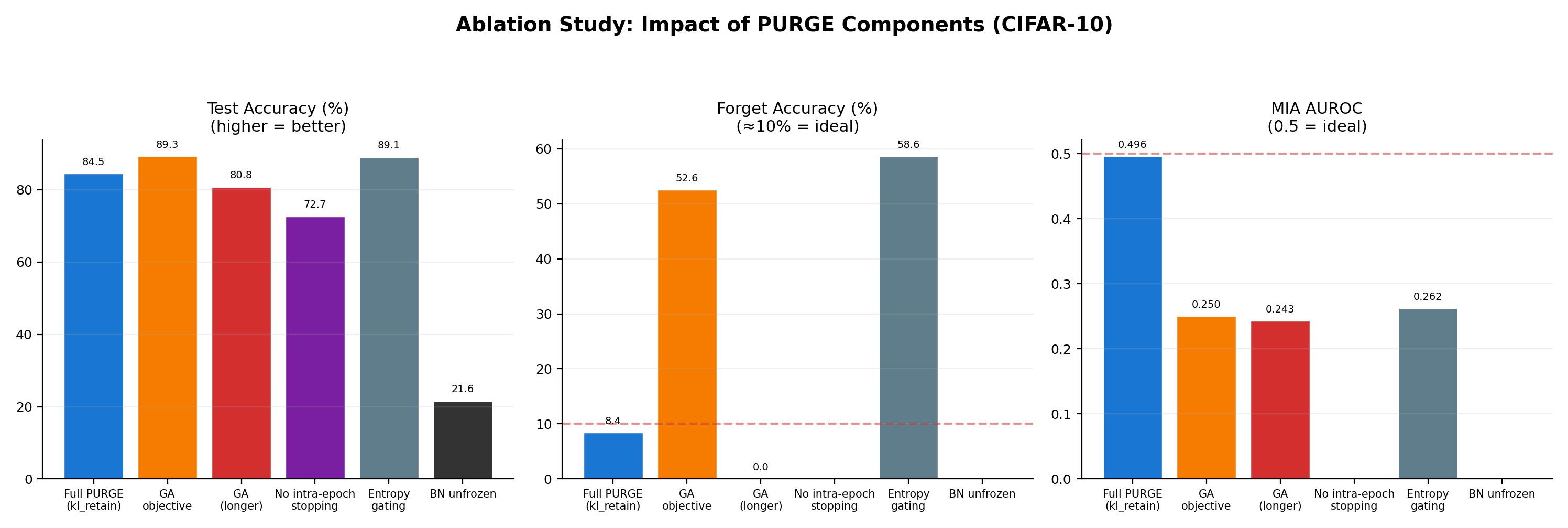}
\caption{Ablation study results. Removing the \klr{} objective, projection, or BN freezing leads to significant degradation.}
\label{fig:ablation}
\end{figure*}

\paragraph{GA vs.\ \klr{}.}
GA either under-forgets (FA = 52.6\% after a few epochs) or over-forgets, reducing retain accuracy (FA = 0.0\%, RA = 93.7\%).
In contrast, \klr{} achieves FA = 8.4\% with RA = 96.3\% and MIA AUROC = 0.496, demonstrating a substantially improved privacy--utility trade-off.

\paragraph{KL-uniform.}
Targeting the uniform distribution achieves FA = 12.0\%, RA = 93.9\%, and MIA AUROC = 0.469.
While this does reduce forget accuracy, the MIA AUROC of 0.469 remains detectably below the ideal 0.5---worse than \klr{}'s 0.496.
The artificial uniform target is detectable by MIA because it does not reflect the natural confusion pattern of a retrained model; \klr{}'s advantage comes from mimicking that natural pattern rather than targeting an artificial distribution.

\paragraph{Entropy gating.}
Setting $\gamma = 0.9$ causes the gate to activate for nearly all batches (98\% gate rate), leading to severe under-forgetting (FA = 58.6\%).
Disabling the gate ($\gamma = 0$) yields a better balance.

\paragraph{BatchNorm freezing.}
Unfreezing BatchNorm during unlearning degrades performance: TA drops by 2.3pp (to 82.2\%) and RA drops by 2.2pp (to 94.1\%).
Class-homogeneous forget batches corrupt the running mean and variance, causing instability in all downstream components.
This failure mode can affect \emph{any} unlearning method operating on networks with BatchNorm when forget batches are class-homogeneous.

\paragraph{Intra-epoch stopping.}
On PathMNIST, FA drops from 99.8\% to 0.0\% within a single epoch.
Intra-epoch FA checks (every 50 batches) detect this transition and stop at FA = 7.3\% (RA = 98.5\%), compared to FA = 0.0\% (RA = 96.8\%) without it.

\section{Feature-Space Fr\'{e}chet Distance}\label{sec:fid}

Standard FID~\cite{heusel2017gans} measures the distributional distance between generated and real images using Inception network features. We adapt this metric to the unlearning setting. Rather than comparing generated images to real images, we compute the Fr\'{e}chet distance between \emph{penultimate-layer features} of the unlearned model and a retrained-from-scratch model, both evaluated on the forget set:

\begin{equation}
    \text{FID}_{\text{feat}} = \|\mu_u - \mu_r\|^2 + \text{Tr}\Big(\Sigma_u + \Sigma_r - 2(\Sigma_u \Sigma_r)^{1/2}\Big)
    \label{eq:fid}
\end{equation}

where $(\mu_u, \Sigma_u)$ are the mean and covariance of the penultimate-layer features extracted from the \emph{unlearned model} on the forget set, and $(\mu_r, \Sigma_r)$ are the corresponding statistics from a \emph{model retrained from scratch} on $\Dr$ only.

\textbf{Interpretation.} A lower $\text{FID}_{\text{feat}}$ indicates that the unlearned model's internal representations on the forget set are closer to those of a model that genuinely never saw the forget data---a stronger form of erasure verification than output-level metrics alone.

\begin{table}[t]
\centering
\small
\caption{Representation-level erasure metrics on CIFAR-10 (class 0). $\text{FID}_{\text{feat}}$: lower = closer to retrained model.}
\label{tab:fid}
\begin{tabular}{lccc}
\toprule
Method & FID$_{\text{feat}}$ $\downarrow$ & Cos $\downarrow$ & L2 $\uparrow$ \\
\midrule
\ours{} (\klr{}) & 195.660 & 0.742 & 16.372 \\
\ours{} (GA) & 102.300 & 0.841 & 12.172 \\
\ours{} (KL-uni) & 412.800 & 0.430 & 21.393 \\
\bottomrule
\end{tabular}
\end{table}

To complement the distributional metrics in \Cref{tab:fid}, \Cref{fig:representation} shows the feature cosine similarity between pre- and post-unlearning representations across all five datasets.
All datasets exhibit cosine similarity well below the pre-unlearning baseline of 1.0, with MNIST (0.637) and PathMNIST (0.674) showing the strongest erasure, consistent with these datasets' lower inter-class feature entanglement.
This confirms that \ours{}'s multi-layer representation erasure removes information from intermediate layers, not only from the output.

\begin{figure}[t]
\centering
\includegraphics[width=\columnwidth]{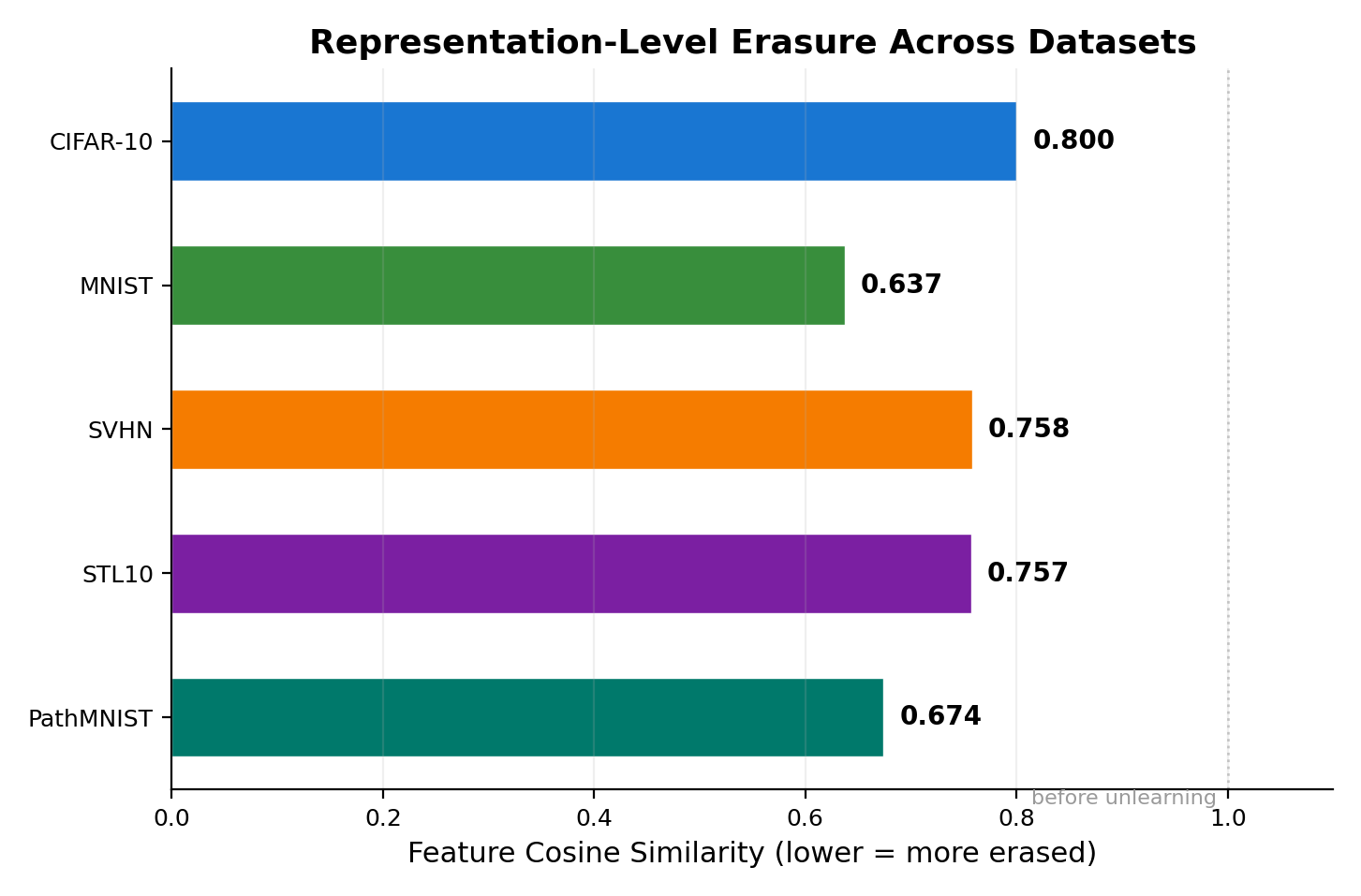}
\caption{Feature cosine similarity (penultimate layer) between pre- and post-unlearning representations across all five datasets. Lower values indicate greater representational change. The dotted line at 1.0 marks the pre-unlearning baseline.}
\label{fig:representation}
\end{figure}

\section{Failure Cases and Limitations}\label{sec:limitations}

We highlight several limitations of \ours{}.

\paragraph{TA gap.}
\ours{} achieves 84.5\% test accuracy (TA) compared to 93.5\% for SalUn on CIFAR-10, a gap of approximately 9 percentage points.
Part of this gap is structural: forgetting all 5{,}000 samples of a class removes shared low-level features that benefit other classes.
Additionally, the \klr{} objective is conservative by design, trading raw accuracy for improved privacy (MIA AUROC close to 0.5).
Whether this trade-off is acceptable depends on the application.
In privacy-sensitive settings (e.g., medical data or regulatory compliance), reduced leakage may justify lower accuracy, whereas in low-risk settings the accuracy advantage of SalUn may be preferable.

\paragraph{Momentum leakage.}
Unlearning is performed using SGD with momentum.
While gradient projection is applied to the raw gradient, the optimizer update includes a momentum term that is not explicitly projected.
As a result, small retain-loss increases may propagate through the momentum buffer.
In practice, the retain-loss budget mitigates this effect; however, the per-step guarantee strictly holds only for vanilla SGD without momentum.
Extending the projection to the full update or bounding the momentum-induced error remains future work.

\paragraph{No formal $(\varepsilon, \delta)$-DP guarantee.}
The projection provides a constructive per-step guarantee (retain loss does not increase), but this does not constitute a formal differential privacy guarantee.
Establishing a connection to $(\varepsilon, \delta)$-DP---potentially via R\'enyi DP accounting---remains an open direction.

\paragraph{PathMNIST TA gap.}
On PathMNIST, TA averages 81.2\% compared to 91.3\% for the base model.
Classes with highly distinctive features (e.g., adipose tissue, epithelium) are more affected, as the \klr{} target assigns them near-zero probability, discouraging correct predictions.
A class-conditional or softened \klr{} target may alleviate this issue.

\paragraph{MIA per-class variation.}
MIA AUROC varies across classes (e.g., MNIST class 1 reaches 0.653), and several PathMNIST classes deviate from the ideal value of 0.5.
This likely reflects differences in class-specific memorisation, although this hypothesis is not explicitly verified.

\paragraph{Sequential unlearning.}
All experiments consider single-class forgetting.
Extending the method to handle sequential forget requests (e.g., multiple classes over time) without cumulative performance degradation is not addressed.

\paragraph{Single-seed results.}
Due to computational constraints, most experiments are conducted with a single random seed, limiting statistical significance.

\section{Discussion}\label{sec:discussion}

\paragraph{The CL--MU connection.}
PURGE is, to our knowledge, the first algorithm to directly instantiate a continual learning mechanism — specifically A-GEM's gradient projection — as the primary update rule for machine unlearning. Prior work, including SCRUB ~\cite{kurmanji2023unbounded}, draws on continual learning as a conceptual motivation but does not directly adapt a CL algorithm's core update step. The distinction matters: SCRUB uses a teacher-student framework inspired by CL intuitions, whereas PURGE inherits A-GEM's projection rule with a direct formal correspondence between the CL constraint (do not increase old-task loss) and the unlearning constraint (do not increase retain-set loss).

Our results suggest that continual learning techniques can be systematically adapted for unlearning, opening a broader design space.
For instance, methods based on EWC~\cite{golatkar2020eternal} or PackNet-style pruning could potentially be adapted for unlearning in a similar manner.

\paragraph{Why we abandoned IEWPv2.}
We explored IEWPv2 (Influence-Weighted Entropy Unlearning Protocol v2), which combined influence-weighted entropy maximisation with interleaved retain training.
Although it achieved reasonable forget accuracy in single-class experiments, two fundamental limitations emerged.
First, IEWPv2 relied on gradient norms as a proxy for influence scores rather than computing the true inverse-Hessian-vector product, weakening its theoretical grounding.
Second, its alternating forget--retain optimisation steps led to partial cancellation: the retain step counteracted forgetting, and vice versa, resulting in slow and unstable convergence that worsened with additional classes.
In contrast, \ours{} addresses both issues by jointly handling forget and retain objectives through a single projected update.

\paragraph{Why \klr{} matters.}
The \klr{} objective consistently produces MIA AUROC values closer to 0.5 than alternative objectives.
This is important because standard unlearning evaluation often focuses on forget accuracy without verifying whether the model's output distribution matches that of a retrained model.
A model with FA = 0\% but MIA AUROC = 0.2 has not fully unlearned, as its outputs remain distinguishable from those of a model trained without the forget data.
The uniform distribution is a commonly used target, but it is artificial and rarely produced by trained models, making it easier for MIA to detect.

\paragraph{Generality across datasets.}
Our experiments cover 22 class-level forgetting tasks across five datasets without architectural changes.
The only adjustments required are modest hyperparameter tuning (e.g., learning rate and KD weight).
The strong performance on PathMNIST is particularly notable, given the importance of unlearning in medical imaging contexts where privacy constraints are stringent.

\paragraph{Future directions.}
Several directions could further strengthen this work.
First, the momentum issue (\Cref{sec:limitations}) could be addressed by projecting the full optimizer update or bounding the momentum-induced error.
Second, extending to sequential unlearning would require mechanisms to track accumulated projection directions across multiple requests.
Third, evaluating on larger-scale datasets (e.g., Tiny ImageNet) and alternative architectures (e.g., Vision Transformers) would test generality beyond ResNet-18.
Finally, establishing a connection between gradient projection and formal $(\varepsilon, \delta)$-DP guarantees would bridge the gap between constructive and statistical notions of privacy.

\section{Ethical and Societal Considerations}\label{sec:ethics}

Machine unlearning directly addresses the ethical imperative of data sovereignty: individuals should retain control over their personal data even after it has been used for model training. \ours{} contributes to this goal by providing an efficient mechanism for class-level data removal with strong empirical privacy guarantees.

\paragraph{GDPR Compliance.} While approximate unlearning methods like \ours{} do not provide the formal certifiability that regulators may ultimately require, they represent a practical middle ground between the infeasibility of full retraining and the unacceptability of ignoring deletion requests. The MIA AUROC $\approx$ 0.5 result provides empirical evidence that the unlearning is effective from an adversarial standpoint.

\paragraph{Potential Misuse.} Unlearning technology could potentially be misused to selectively remove evidence of model bias or to undermine forensic analysis of model provenance. We note that \ours{}'s unlearning leaves auditable artifacts (the retain-loss budget, projection rate logs, and FA trajectory) that could support post-hoc verification.

\paragraph{Fairness Considerations.} Selective unlearning of an entire class may disproportionately affect model performance on visually similar classes. The TA degradation observed in \ours{} reflects this concern and warrants careful monitoring in deployment.

\FloatBarrier
\section{Conclusion}\label{sec:conclusion}

We presented \ours{}, a machine unlearning algorithm motivated by the duality between continual learning and machine unlearning.
By adapting gradient projection from continual learning, \ours{} enforces that each update does not increase retain-set loss (with the caveat that the guarantee applies to raw gradients rather than momentum-augmented updates).
Combined with the retain-confusion target, multi-layer representation erasure, and dual stopping criteria, \ours{} achieves MIA AUROC close to 0.5 across five datasets while maintaining retain accuracy above 96\%.

\ours{} does not fully resolve machine unlearning: it exhibits a test accuracy gap relative to some baselines, lacks formal differential privacy guarantees, and has only been evaluated at the ResNet-18 scale under limited computational resources.
Nevertheless, the CL--MU duality perspective provides a principled framework for algorithm design, and the empirical findings (particularly regarding BatchNorm behaviour and evaluation metrics) offer practical guidance for future work.

All experiments are conducted using ResNet-18 on a single GPU due to computational constraints; scaling to larger models and datasets remains an important direction for future research.


\bibliography{references}

\appendix

\section{Full Per-Class Results}
\label{app:perclass}

\subsection{MNIST (10 Classes)}

\begin{table}[h]
\centering
\caption{\ours{} (\klr{}) per-class unlearning on MNIST.}
\label{tab:mnist_full}
\small
\begin{tabular}{ccccc}
\toprule
Forget Cls & TA (\%) & FA (\%) & RA (\%) & MIA \\
\midrule
0 & 89.99 & 1.99 & 99.82 & 0.512 \\
1 & 88.30 & 5.38 & 99.87 & 0.653 \\
2 & 89.99 & 7.67 & 99.87 & 0.510 \\
3 & 90.77 & 9.28 & 99.88 & 0.510 \\
4 & 90.68 & 6.09 & 99.87 & 0.511 \\
5 & 91.24 & 7.38 & 99.88 & 0.521 \\
6 & 90.97 & 7.06 & 99.85 & 0.495 \\
7 & 90.25 & 8.14 & 99.90 & 0.522 \\
8 & 90.66 & 9.38 & 99.84 & 0.512 \\
9 & 90.01 & 5.31 & 99.91 & 0.524 \\
\midrule
\textbf{Avg} & \textbf{90.19} & \textbf{6.77} & \textbf{99.87} & \textbf{0.517} \\
\bottomrule
\end{tabular}
\end{table}

\subsection{PathMNIST (9 Classes)}

\begin{table}[h]
\centering
\caption{\ours{} (\klr{}) per-class unlearning on PathMNIST.}
\label{tab:pathmnist_full}
\small
\begin{tabular}{clcccc}
\toprule
Cls & Tissue Type & TA & FA & RA & MIA \\
\midrule
0 & Adipose & 74.35 & 7.28 & 98.45 & 0.545 \\
1 & Background & 79.35 & 9.85 & 98.33 & 0.614 \\
2 & Debris & 89.30 & 9.47 & 98.93 & 0.369 \\
3 & Lymphocytes & 83.20 & 10.58 & 99.00 & 0.425 \\
4 & Mucus & 77.49 & 9.44 & 98.53 & 0.517 \\
5 & Smooth Musc. & 84.36 & 9.45 & 98.55 & 0.501 \\
6 & Normal Muc. & 80.93 & 9.87 & 98.81 & 0.594 \\
7 & Stroma & 87.16 & 10.71 & 99.39 & 0.521 \\
8 & Epithelium & 75.96 & 7.12 & 98.74 & 0.557 \\
\midrule
\multicolumn{2}{c}{\textbf{Average}} & \textbf{81.24} & \textbf{9.31} & \textbf{98.74} & \textbf{0.516} \\
\bottomrule
\end{tabular}
\end{table}

\section{Reproducibility}

Code, configs, and checkpoints are available at: \url{https://github.com/vedjaw/la_purge} 

\textbf{Key implementation details:}
\begin{itemize}
    \item All experiments use PyTorch $\geq$ 2.0 with CUDA.
    \item YAML config files for each dataset are provided in \texttt{configs/}.
    \item The retain confusion distribution is precomputed once before unlearning and stored in memory.
    \item Forward hooks are registered on \texttt{layer3}, \texttt{layer4}, and \texttt{avgpool} for ResNet-18.
    \item BatchNorm layers are set to \texttt{eval()} mode immediately after setting the model to \texttt{train()} mode, ensuring trainable parameters receive gradients while running statistics remain frozen.
\end{itemize}

\end{document}